\documentclass{article}

\PassOptionsToPackage{numbers, compress}{natbib}

\usepackage[preprint]{neurips_2023}

\usepackage{booktabs}       
\usepackage{amsfonts}       
\usepackage{amsmath,amssymb}
\usepackage{mathtools}

\usepackage{amsmath,amsfonts,bm}

\def\eqref#1{equation~\ref{#1}}

\def\1{\bm{1}}

\DeclareMathAlphabet{\mathsfit}{\encodingdefault}{\sfdefault}{m}{sl}
\SetMathAlphabet{\mathsfit}{bold}{\encodingdefault}{\sfdefault}{bx}{n}

\usepackage{hyperref}
\usepackage{url}
\usepackage{graphics}
\usepackage{graphicx}
\usepackage{subfigure}
\usepackage{caption}
\usepackage{wrapfig}
\usepackage{amsthm}
\usepackage{multirow}
\usepackage{multicol}
\usepackage[ruled,vlined]{algorithm2e}
\usepackage{bbding}
\newtheorem{theorem}{Theorem}
\newtheorem{lemma}[theorem]{Lemma}

\newtheorem{proposition}{Proposition}

\newcommand{\xx}{\boldsymbol{x}}

\newcommand{\diff}{\mathrm{d}}
\newtheorem{definition}[theorem]{Definition}
\newtheorem{remark}{Remark}
\newtheorem*{theorem*}{Theorem}
\newtheorem*{proposition*}{Proposition}

\title{Training Energy-Based Models with Diffusion Contrastive Divergences}

\author{%
    Weijian Luo$^{1}$\thanks{This work was done when he was a research intern at Huawei Noah's Ark Lab. Email: luoweijian@stu.pku.edu.cn.},~~Hao Jiang$^{3}$\thanks{This work was done when he was a research intern at Huawei Noah's Ark Lab},~~Tianyang Hu$^{2}$,~~Jiacheng Sun$^{2}$\thanks{Corresponding to: Jiacheng Sun (sunjiacheng1@huawei.com)},~~Zhenguo Li$^{2}$,~~Zhihua Zhang$^{1}$\\[1em]
    $^1$Peking University, $^2$Huawei Noah's Ark Lab,  $^3$Harbin Institute of Technology (Shenzhen)
}

\begin{document}

\maketitle

\begin{abstract}
Energy-Based Models (EBMs) have been widely used for generative modeling. 
Contrastive Divergence (CD), a prevailing training objective for EBMs, requires sampling from the EBM with Markov Chain Monte Carlo methods (MCMCs), which leads to an irreconcilable trade-off between the computational burden and the validity of the CD. Running MCMCs till convergence is computationally intensive. 
On the other hand, short-run MCMC brings in an extra non-negligible parameter gradient term that is difficult to handle. In this paper, we provide a general interpretation of CD, viewing it as a special instance of our proposed Diffusion Contrastive Divergence (DCD) family.
By replacing the Langevin dynamic used in CD with other EBM-parameter-free diffusion processes, we propose a more efficient divergence. We show that the proposed DCDs are both more computationally efficient than the CD and are not limited to a non-negligible gradient term. We conduct intensive experiments, including both synthesis data modeling and high-dimensional image denoising and generation, to show the advantages of the proposed DCDs. On the synthetic data learning and image denoising experiments, our proposed DCD outperforms CD by a large margin. In image generation experiments, the proposed DCD is capable of training an energy-based model for generating the Celab-A $32\times 32$ dataset, which is comparable to existing EBMs.
\end{abstract}

\section{Introduction}
Energy-Based Models (EBMs) are an important part of unsupervised learning \citep{lecun2006tutorial, Hinton06, Zhu2004FiltersRF}. Paired with the superb expressive power of deep neural networks, EBMs draw great attention in the machine learning community and have broad applications in many unsupervised learning tasks such as generative modeling \citep{xie2016theory,gao2020learning,nijkamp2019learning, Zhao2021LearningEG,du2019implicit,grathwohl2019your}, out-of-distribution detection \citep{zhai2016deep, liu2020energy, Lee2020AdversarialTO}, concept learning \citep{mordatch2018concept, Du2020CompositionalVG} and others \citep{Haarnoja2017ReinforcementLW, Xie2017SynthesizingDP, Xie2018LearningDN, Ingraham2019LearningPS}. 
Despite the popularity, the training of EBMs is challenging and remains an active field of research. 
One dominant line of training methods of EBMs relies on sampling from the EBMs by running MCMC chains \citep{song2021train,hinton2002training,Hinton06,du2019implicit,du2020improved,gao2020learning,grathwohl2019your}, whose convergence can be computationally expensive in practice. 
To improve efficiency, \cite{hinton2002training} proposed the Contrastive Divergence (CD), which was calculated via \emph{short-run} MCMC chains that are initialized from data samples. 
An overview of CD can be seen in Figure \ref{fig:overall}(a), where the data distribution is transported with EBM-induced MCMCs as the upper line of the figure illustrates.
The CD was further developed in many works \citep{tieleman2009using,xie2022tale,du2019implicit,nijkamp2019learning,miyato2018spectral,du2019implicit,grathwohl2019your} and has become a general approach for training EBMs. 

Nonetheless, the CD has its own drawbacks that are deeply rooted in the employed MCMC mechanism. To be more specific, samples from MCMCs are induced by EBMs, so these samples depend on EBMs' parameters, leading to a non-negligible gradient term that is difficult to handle as we introduced in Section \ref{sec:background}. Some works overlooked the parameter dependence for simplicity \citep{hinton2002training,liu2017learning}. As pointed out by \citet{du2020improved}, such an omission leads to training failures, e.g., non-convergence of training objectives. 
To address the parameter-dependence issue, \citet{du2020improved} proposed to consider the non-negligible gradient term through an additional non-parametric entropy estimation component. However, the non-parametric entropy estimation is neither efficient nor scalable for high-dimensional data. 

\begin{figure*}
\centering
\subfigure[Contrastive Divergence]{\includegraphics[width=0.45\textwidth
    ]{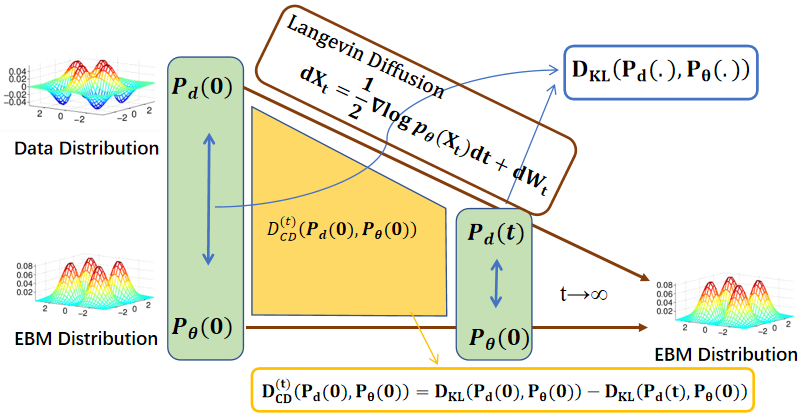}}
\hspace{5mm}
\subfigure[Diffusion Contrastive Divergence]{\includegraphics[width=0.45\textwidth
    ]{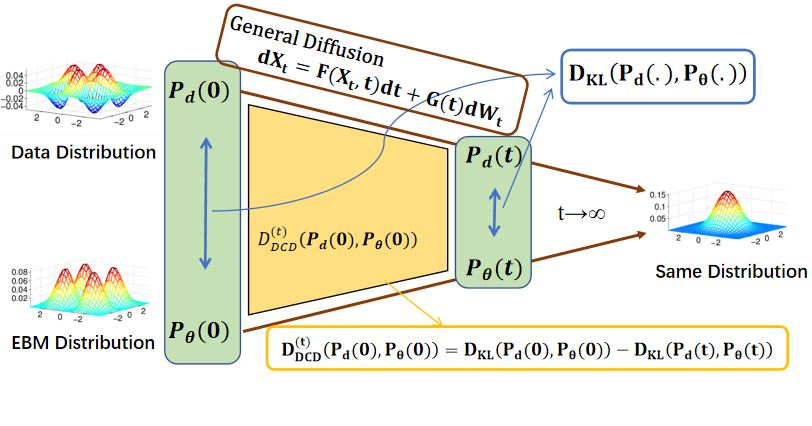}}
\caption{Illustration of DCD and CD. The yellow area represents the corresponding divergence. The CD takes the EBM-induced Langevin dynamics to transport data and EBM distribution to meet with the same EBM distribution. The DCD considers a more general diffusion process to transport both data and EBM distribution to meet with the same distribution.}
\label{fig:overall}
\end{figure*}

In this work, we address the parameter-dependence issue of CD by extending the Langevin diffusion, a commonly used MCMC for CD, to general diffusion processes and propose a novel family of divergences --- the \emph{diffusion contrastive divergence} (DCD) family, as illustrated in Figure \ref{fig:overall}(b). 
Our proposed DCD family is both theoretically sound and computationally efficient. The contributions of our proposed DCD is three folded. First, the DCD overcomes the non-negligible gradient issue of CD that influence the accuracy of CD. Second, the DCD does not depend on EBM-induced MCMC so is efficient when implemented. Third, the proposed DCD framework provides a unified view that includes the CD as a special instance. The framework can potentially benefit further understanding and developing algorithms for training EBMs. To demonstrate the effectiveness and efficiency of the proposed DCD, we instantiate the DCD with a special VE diffusion process and call it the DCD-VE (or just DCD for short) algorithm. We conduct experiments with DCD-VE in three experiments including synthetic data modeling, image denoising, and image generation. 
On the synthetic data learning and high-dimensional image denoising experiments, the proposed DCD-VE outperforms CD with a significant margin. On the image generation experiment, we train a time-dependent energy-based model on the CelebA dataset of a resolution of $32\times 32$. The trained EBM is comparable to previous EBMs on generation.
Besides, the experiments demonstrate that the DCD is more efficient than CD, being 2-4 times faster in terms of the wall-clock time.

\section{Background}\label{sec:background}
\paragraph{Energy-based models.}
Let $p_d$ represent the data distribution. An energy-based model specifies the density with a neural-parametrized energy function $f_\theta(\xx)$ with the form
\begin{align}\label{equ:ebm_define}
    p_\theta(\xx) = \frac{\exp(f_\theta(\xx))}{Z_\theta},
\end{align}
where $f_\theta$ is usually a deep neural network and $Z_\theta = \int \exp(f_\theta(\bm{u}))d\bm{u}$ is the unknown normalizing constant. To make the derivation neat, we slightly abuse the conventions and call $f_\theta(\xx)$ the energy function. In most cases, $Z_\theta$ is so complicated that is intractable, making the likelihood intractable as well. Previous works find out that the difficulty of estimating the normalizing constant can be circumvented with consistent sampling from the EBM when training with Maximum Likelihood Estimation (MLE). More precisely, the derivative of EBM's expected likelihood over data distribution has an expression
\begin{align}
    \frac{\partial}{\partial \theta}\mathbb{E}_{p_d}\log  \frac{\exp(f_\theta(\xx))}{Z_\theta}  = \mathbb{E}_{ p_d}\frac{\partial}{\partial \theta} f_\theta(\xx) - \mathbb{E}_{p_\theta}\frac{\partial}{\partial \theta}f_\theta(\xx).
\end{align}
This expression shows that the likelihood function's parameter gradient can be estimated with samples consistently drawn from data and the EBM. The Langevin dynamics (LD), is a usual choice of MCMC for obtaining samples from EBMs. It simulates the diffusion process
\begin{align}\label{equ:ld}
    \diff \xx_t = \frac{1}{2}\nabla_{\xx_t}\log p_\theta(\xx_t) \diff t + \diff \bm{w}_t,
\end{align}
in order to draw samples from the EBM. Under mild conditions \citep{pavliotis2014stochastic}, the marginal distribution of \eqref{equ:ld} will converge to the target distribution regardless of the initial distribution. Here $\bm{w}_t$ is an independent Wiener process. Notice that the normalizing constant $Z_\theta$ in Equation \eqref{equ:ebm_define} is independent of $\xx$, so we have 
\begin{align*}
    \nabla_{\xx_t} \log p_\theta(\xx_t) \coloneqq \nabla_{\xx_t} \big[ f_\theta(\xx_t) + Z_\theta\big] = \nabla_{\xx_t} \log f_\theta(\xx_t).
\end{align*}
This shows that the LD can take EBMs' neural network without the influence of the unknown normalizing constant. 
\paragraph{Contrastive divergence and the non-negligible gradient term.}
Training EBMs with MLE requires MCMC chains to run sufficiently long so as to draw samples from the EBM. Some works studied the possibility of training EBMs with un-converged MCMCs. \citet{hinton2002training} and \citet{Hinton06} observed that a few MCMC steps which are initialized from data samples work well empirically so they argued the MCMC chains do not need to fully converge when training EBMs. They thus formally proposed the Contrastive Divergence as
\begin{align}\label{eqn:cd_0}
    \mathcal{D}_{CD}(p_d,p_\theta) = \mathcal{D}_{KL}(p_d, p_\theta) - \mathcal{D}_{KL}(p_{d,\theta}^{(T)}, p_\theta),
\end{align}
where $p_{d,\theta}^{(T)}$ stands for the marginal distribution of a short-run MCMC initialized from $p_d$ with transition time $T$ and the notation $\mathcal{D}_{KL}$ denotes the Kullback–Leibler (KL) divergence. For such a definition, the non-negativity $\mathcal{D}_{CD}(p,q)\geq 0$ holds and $\mathcal{D}_{CD}(p_d, p_\theta)=0$ only when $p_t=q_t$ almost everywhere. This makes $\mathcal{D}_{CD}$ a reasonable divergence, we put detailed derivation on the non-negativity of CD in the Appendix. If we take the parameter derivative, we have
\begin{align}\label{eq:cd_grad}
    &\frac{\partial}{\partial \theta}\mathcal{D}_{CD}(p_d,p_\theta) = \mathbb{E}_{ p_{d,\theta}^{(T)}}\big[\frac{\partial}{\partial \theta}f_\theta(\xx) \big] -\mathbb{E}_{ p_d}\big[\frac{\partial}{\partial \theta} f_\theta(\xx)\big] - \mathbb{E}_{p_{d,\theta}^{(T)}}\bigg[\log p_\theta(\xx) \frac{\partial}{\partial \theta}\log p_{d,\theta}^{(T)}(\xx) \bigg].
\end{align}
The third gradient term is difficult to handle because, for EBM, we do not know the value of normalizing constant $Z_\theta$. So the gradient $\frac{\partial}{\partial\theta} \log p_{d,\theta}^{(T)}(\xx)$ is also unknown. \citet{hinton2002training} and \citet{liu2017learning} proposed to omit the third gradient term and simplify the CD \eqref{eq:cd1} as
\begin{equation}\label{eq:cd1}
    \mathbb{E}_{\xx_T\sim \operatorname{sg}[p_{d,\theta}^{(T)}]}f_\theta(\xx_T) - \mathbb{E}_{\xx\sim p_d} f_\theta(\xx).
\end{equation}
Here the notation $\xx_T\sim \operatorname{sg}[p_{d,\theta}^{(T)}]$ represents the sample $x_T$ is drawn from $p_{d,\theta}^{(T)}$ but omitting the parameter dependence of $\theta$.
In practice, there is always a non-negligible term for the gradient of the contrastive divergence. \citet{du2020improved} tried to address the non-negligible term by introducing an additional non-parametric entropy estimation component together with the training of EBM, viewing the non-negligible third term of \eqref{eq:cd_grad} as a parameter derivative of Shannon entropy that is estimated non-parametrically. Although technically sound, the entropy estimation which \citet{du2020improved} brought in is computationally intensive and not scalable in high dimensions. 

\paragraph{Diffusion process.}
A diffusion process is a stochastic process driven by a stochastic differential equation (SDE) \citep{Srkk2019AppliedSD} with a drift vector $\bm{F}$ and a diffusion matrix $\bm{G}$,
\begin{equation}\label{form:1}
    \diff \xx_t = \bm{F}(\xx_t,t)\diff t + \bm{G}(t)\diff \bm{w}_t,
\end{equation}
where $\bm{w}_t$ is a standard Wiener process. For simplicity, we assume $\bm{G}$ to be a scalar function of time $t$ in the rest of the paper. If a diffusion process is initialized with an initial distribution $p_0$, then the evolution of marginal probability density is governed by the Fokker-Planck equation \citep{Risken1984FokkerPlanckE}:
\begin{align}\label{equ:fp_equation}
  \frac{\diff}{\diff t}p(\xx,t) &= - \langle \nabla_{\xx}, p(\xx,t)\bm{F}(\xx,t)\rangle + \frac{1}{2}\bm{G}^2(t)\Delta_{\xx} p(\xx,t), \ p(\xx,0) = p_0(\xx).
\end{align}
The Langevin dynamics defined in \eqref{equ:ld} is an instance of diffusion processes. The VE diffusion is a commonly used diffusion process in generative modeling \citep{song2020score,song2021maximum,karras2022edm}. It writes 
\begin{align}
    \label{eq:ve_forward}
    & \diff \xx_t = g(t)\diff \bm{w}_t.
\end{align}
The diffusion has explicit conditional distributions $p_t(\xx_t|\xx_0)$ and their marginal samples are cheap to obtain as we put in the Appendix. 

\section{Diffusion contrastive divergences}
Our goal is to propose novel training methods that overcome both the non-negligible gradient term and the inefficiency issue caused by MCMC of CD, by generalizing the definition of CD to other parameter-free diffusion processes, named diffusion contrastive divergence (DCD). In this section, we first give the formal definition of DCD. Then we establish the connections of DCD to existing methods, namely the diffusion recovery likelihood and the KL-contraction divergence. Later we proposed a practical algorithm, the DCD-VE based on the VE diffusion \eqref{eq:ve_forward} for training energy-based models. 

\subsection{CD with general diffusions}
We follow the notations defined in Section \ref{sec:background} and take the LD as the MCMC which defines the CD. Recall the definition of CD \eqref{eq:cd1}.

One of the most important reasons for taking LD to define the divergence is that the KL divergence of the marginal distributions with LD is strictly decreasing and converges to $0$ when $T\to\infty$ unless $p_\theta = p_d$ (We put in Appendix). This makes the CD a well-defined divergence. But the definition of LD \eqref{equ:ld} incorporates the EBM and its parameters $\theta$, giving rise to a hard-to-handle non-negligible gradient term as we pointed out in \ref{sec:background}. Besides, obtaining samples with LD also relies on the sequential simulation of SDE which is computationally inefficient. So it would be ideal if the Langevin dynamics that the CD uses are replaced with some parameter-free alternatives.

Fortunately, other diffusion processes, such as the VE process \eqref{eq:ve_forward} with the properly defined function $g(t)$ also guarantee the strict decrease and the convergence of the KL between marginal distributions as LD does. Besides, the definition of such diffusion processes does not contain any EBM parameters, and the marginal samples are efficient to obtain as we put in discussions in Appendix. 

Based on such an observation, we formally define the \emph{Diffusion Contrastive Divergence} (DCD), as the KL difference between an initial distribution and the transitional distribution under some pre-defined diffusion process.
\begin{definition}[Diffusion Contrastive Divergence]
    \begin{align}\label{def:dcd}
    \mathcal{D}_{DCD}^{(\bm{F},\bm{G},T)}(p_d, p_\theta) := \mathcal{D}_{KL}(p_d, p_\theta) - \mathcal{D}_{KL}(p_d^{(T)}, p_\theta^{(T)}).
\end{align}
Here $p_d^{(T)}$ and $p_\theta^{(T)}$ stand for the marginal distributions of the diffusion \eqref{form:1} that are initialized with $p_d$ and $p_\theta$ respectively.
\end{definition}
To further study the properties of the proposed DCDs, we first give a theorem to verify that the DCD is a well-defined probability divergence. 
\begin{theorem}\label{thm:1}
Let $\bm{F}(\xx,t)$ and $\bm{G}(t)$ be two pre-defined functions. For two distributions $p$ and $q$, assume both $p,q$ evolve according to the same diffusion process \eqref{form:1}.
Let $p^{(t)}$ and $q^{(t)}$ denote the time $t$ marginal distribution under SDE evolution. Then we have 
\begin{align*}
    &\mathcal{D}_{DCD}^{(\bm{F},\bm{G},T)}(p,q) = \frac{1}{2}\int_{0}^T \mathbb{E}_{\xx_t \sim p^{(\bm{F},\bm{G},t)}(x)} \bm{G}^2(t)\|\nabla_{\xx_t} \log p^{(\bm{F},\bm{G},t)}(\xx_t) - \nabla_{\xx_t} \log q^{(\bm{F},\bm{G},t)}(\xx_t)\|^2_2 \diff t.
\end{align*}
\end{theorem}
We give detailed proof in the Appendix. From Proposition \ref{thm:non_negativity}, we see that DCD is non-negative.
\begin{proposition}\label{thm:non_negativity}
For any two distributions $p$ and $q$, any function $\bm{F},\bm{G}$ and any diffusion time $T$, then 
\[
\mathcal{D}_{DCD}^{(\bm{F},\bm{G},T)}(p,q) \geq 0.
\]
\end{proposition}
With a suitable choice of $\bm{F}$ and $\bm{G}$, the KL divergence between marginal distributions is strictly decreasing, thus the defined $\mathcal{D}^{(\bm{F}, \bm{G}, T)}$ does not degenerate, making $\mathcal{D}^{(\bm{F}, \bm{G}, T)}(p,q) = 0$ if and only if $p=q, a.e.$. As we show in Appendix, the VE diffusion satisfies this property. 

For a diffusion process that does not depends on EBM's parameter, the corresponding DCD avoids the parameter-dependence issues. Figure \ref{fig:overall}(b) gives the concept of DCD. Both the data and EBM's distribution evolve along the diffusion process specified by $(\bm{F},\bm{G})$ as in \eqref{form:1}. With suitable choices, when $T\to \infty$, two involved distributions coincide with the same stationary distribution. The yellow region accounts for what DCD measures. 

\begin{remark}
    Notice that $p_\theta$ itself is a stationary distribution of the above LD as we put in Appendix. Hence, $p_{\theta}^{(t)} = p_\theta$ holds for any $t\in [0,T]$. So if we choose a special $\bm{F}(\xx,t) = \nabla_{\xx} f_\theta(\xx)/2$ and $\bm{G}(t)=\mathbf{I}$, the proposed $\mathcal{D}_{DCD}^{(\bm{F},\bm{G},T)}$ recovers CD (\eqref{eq:cd1}).
\end{remark}

\begin{table*}[h]
\caption{Comparison of DCD and CD.}
\label{tab:1}
\vskip 0.05in
\begin{center}
\begin{scriptsize}
\begin{tabular}{lccccc}
\toprule
Method & MCMC & Diffusion Process  & One-step DCD Formula\\
\midrule
CD &  $\checkmark$ & $\diff\xx_t = -\nabla \frac{f_\theta(\xx_t)}{2}\diff t + \diff \bm{w}_t$  & Stationary\\
DCD-VE & \XSolidBrush  & $\diff\xx_t = g(t)\diff \bm{w}_t$ &Eq.(\ref{eqn:dcd_ve})\\
\bottomrule
\end{tabular}
\end{scriptsize}
\end{center}
\vskip -0.1in
\end{table*}
To be more concrete, we consider VE diffusion as a demonstration. Recall the definition of VE diffusion \ref{eq:ve_forward}.
The conditional distribution of the VE diffusion does not depend on EBM's parameter $\theta$. The marginal samples can be drawn with $\bm{x}_0 \sim p_d, \xx_t\sim p_t(\xx_t|\xx_0)$. 
\begin{theorem}
Minimizing the DCD is equivalent to minimizing the following divergence.
    \begin{align}\label{equ:dcd_tractable}
    \mathcal{L}_{DCD}(\theta) =& \mathbb{E}_{\xx_0\sim p_d, \xx_t\sim p(\xx_t|\xx_0)}\big[f_\theta^{(\bm{F},\bm{G},T)}(\xx_t)\big] - \mathbb{E}_{\xx_0\sim p_d}\big[f_\theta(\xx_0) \big].
\end{align}
Here $f_\theta^{(\bm{F},\bm{G},T)}$ are time $T$ marginal energy under diffusion process (\eqref{form:1}). 
\end{theorem}
Check the Appendix for detailed proof. The term $\log p_d(\xx)$ and $\log p_d^{(\bm{F},\bm{G},T)}$ are independent of parameter $\theta$ since the diffusion process is parameter-free. The equation \eqref{equ:dcd_tractable} defines a tractable objective that is equivalent to DCD.  

The advantages of the DCD with VE diffusion over CD are two-fold. 
First, recall that the CD is hindered by the parameter-dependence of both the transitional distribution $\xx_T\sim p_{d,\theta}^{(T)}$ and the $T$-time evolved data distribution $\log p_{d,\theta}^{(T)}(\xx_T)$ in the MCMC chains. These two terms are parameter-free if we choose a parameter-free diffusion instead of Langevin dynamics. 

Second, the sampling from VE diffusion gets significantly cheaper when taking specially designed diffusions such as VE diffusion. However as a trade-off, one needs to evaluate the time $T$ marginal energy of $f_\theta^{(\bm{F},\bm{G},T)}(\xx_t)$, which can be easier to handle. We provide further analysis of the energy evolution in Section \ref{sec:energy_evolve}.
To summarize, DCD is an MCMC-free method that overcomes the CD's two difficulties with one easier problem of estimating the energy evolution. Such MCMC-free training methods for EBMs are a hot research area in EBM community\citep{grathwohl2021no}. We give a brief summary of the differences between CD and the DCD that are defined through the VE diffusion in Table \ref{tab:1}.

\subsection{Connections to existing methods}
The DCD framework not only provides a new understanding of CD but also more insights into existing works on training EBMs. For instance, DCD has inner connections to two existing methods, the Diffusion Recovery Likelihood \citep{gao2020learning,bengio2013generalized} and the KL-Contraction Divergence \citep{lyu2011unifying}.

\paragraph{Connection to Diffusion Recovery Likelihood.} 
Let $p^{(\sigma)}(\Tilde{\xx}|\xx) = \mathcal{N}(\Tilde{\xx};\xx,\sigma^2\mathbf{I})$ denotes a Gaussian perturbation on $x$. The recovery likelihood of a data $\xx$ is defined as the conditional probability to recover $\xx$ from noise perturbed observation $\tilde{\xx}$, i.e., $p_\theta(\xx|\Tilde{\xx})= p^{(\sigma)}(\Tilde{\xx}|\xx)p_\theta(\xx)/p_\theta(\Tilde{\xx})$, which is proportional to $\exp(f_\theta(\xx) -\frac{1}{2\sigma^2}\|\Tilde{\xx} - \xx\|_2^2)$.

\citet{gao2020learning} viewed recovery likelihood as a new EBM for $\xx$ if $\tilde{\xx}$ is given as fixed and minimized the recovery likelihood through a CD-like MCMC method for which negative samples are consistently sampled from $p_\theta(\xx|\Tilde{\xx})$-induced MCMC. \citet{gao2020learning} also extended the recovery likelihood to multi-level Gaussian noise level $\{\sigma_i\}$ to define a diffusion recovery likelihood. Surprisingly as we show in this section, the recovery likelihood objective is a special case of DCD when taking the diffusion process to be the VE diffusion. Revisit that the definition of the recovery likelihood writes
\[
\mathbb{E}_{\xx\sim p_d,\Tilde{\xx}\sim p_d^{(\sigma)}(\Tilde{\xx})}\log p_\theta(\xx|\Tilde{\xx}),
\]
the $p^{(\sigma)}(\Tilde{\xx}|\xx)$ and $p_d(\xx)$ are independent of parameter $\theta$, so maximizing the recovery likelihood is equivalent to minimizing
\begin{align*}
\mathcal{D}_{KL}(p_d(\xx),p_\theta(\xx)) - \mathcal{D}_{KL}(p_d^{(\sigma)}(\Tilde{\xx}),p_\theta^{(\sigma)}(\Tilde{\xx})).
\end{align*}
We put the detailed derivation in the Appendix. Here $p_\theta^{(\sigma)}(\Tilde{\xx}) = \int p_\theta(\xx)p(\Tilde{\xx}|\xx)d\xx$ is the marginal density of Gaussian perturbed distribution. The recovery likelihood and its diffusion counterpart are special cases of DCD when taking the diffusion process to be VE diffusion \eqref{eq:ve_forward}. When setting $\sigma_i^2 = \int_0^{t_i}g(s)ds$, the DCD-VE recovers the diffusion recovery likelihood. However, the implementation of maximizing recovery likelihood in \citep{gao2020learning} is different. They sample from $\log p_\theta(\xx|\Tilde{\xx})$ through MCMC when training, making the training procedure computationally expensive. In our definition of the DCD, we do not require sampling from recovery likelihood. We instead use contrastive mechanics between $p$ and $p^{(T)}$ to cancel out the normalizing constant as we introduced in  later sections. Besides, the DCD framework can be generalized to other diffusion processes of which the definition does not involve EBM's parameters. 

\paragraph{DCD as a KL-contraction divergence.} \citep{lyu2011unifying} proposed the so-called KL contraction divergence framework. They pointed out that if an operator $\Phi(p)$ satisfies the KL contraction property, meaning $$\mathcal{D}_{KL}(\Phi(p),\Phi(q))\leq \mathcal{D}_{KL}(p,q),$$ a KL-contraction divergence can be defined as $\mathcal{D}_{KL}(p,q) - \mathcal{D}_{KL}(\Phi(p),\Phi(q))$. As we mentioned in the Theorem \ref{thm:1}, the marginalization along any diffusion process is a KL contraction operator, so the DCD can be viewed also as a KL-contraction divergence. However, in our paper, we define the DCD through the motivation of generalizing the CD. Besides, we propose a concrete divergence, the DCD-VE, which is much different from the instances that have been studied in \citet{lyu2011unifying}.

\subsection{Evolution of the energy function}\label{sec:energy_evolve}
Since the definition of DCD \eqref{def:dcd} involves the computation of the diffused density function $p_{d,\theta}^{(T)}(\xx)$ and corresponding energy function $f_\theta^{(T)}(\xx)$, so in this section, we characterize the evolution of the energy function $f_\theta^{(T)}(\xx)$ through a partial differential equation. Denote $p_\theta^{(0)}(\xx) = e^{f_\theta(\xx)}/Z_\theta$ where $Z_\theta$ is the normalizing constant. We show that the evolution of the energy function under the diffusion process (\eqref{form:1}) follows a PDE.

\begin{proposition}
Assume $p_\theta^{(0)}(\xx) = e^{f_\theta(\xx)}/Z_\theta$ where $Z_\theta$ is a parameter-dependent normalizing constant. Assume $p_\theta^{(t)}$ denotes the evolved density along a diffusion process \eqref{form:1}, then for any fixed $\xx$, the energy value $p_\theta^{(t)}(\xx)$ evolves according to a PDE
\[
\diff \log p_\theta^{(t)}(\xx)/\diff t=\mathcal{O}(\nabla_{\xx} \log p_\theta^{(t)}),
\]
where $\mathcal{O}(\nabla_{\xx} \log p_\theta^{(t)})$ is the following operator which is independent of the normalizing constant, 
\begin{align*}
    &\langle \bm{G}^2(t)\nabla_{\xx} \log p_\theta^{(t)}(\xx)/2 - \bm{F}(\xx,t), \nabla_{\xx} \log p_\theta^{(t)}(\xx) \rangle + \langle \nabla, \bm{G}^2(t)\nabla_{\xx} \log p_\theta^{(t)}(\xx)/2 - \bm{F}(\xx,t) \rangle.
\end{align*}
\end{proposition}
It is worth emphasizing that since the evolution operator $\mathcal{O}(.)$ does not depend on $Z_\theta$, the normalizing constant keeps unchanged in the process and thus will be exactly canceled out when we substitute the $T$-time KL and initial KL as in DCD expression. So the DCD is not bothered by a parameter-dependent normalizing constant. We give a more detailed argument in the Appendix.

In practice, we do not need many steps when training EBM. So we use a single step as an approximation when implementing DCD. Our experiments show that the single-step DCD works well in practice. Here we derive a one-step approximation of DCD for practical implementations. 

\paragraph{DCD-VE.} For VE diffusion \eqref{eq:ve_forward}, the time-change rate of energy can be approximated with 
\begin{align}\label{eqn:dcd_ve}
    \mathcal{L}_{DCD}^{(VE)}(\theta) = \mathbb{E}_{p_t}\frac{1}{2}\bm{G}^2(0)\bigg[ \|\nabla_{\xx} f_\theta(\xx_t)\|^2 + \Delta f_\theta(\xx_t) \bigg] +\frac{1}{t} \bigg[ \mathbb{E}_{p_t}[f_\theta(\xx_t)] - \mathbb{E}_{p_d}[f_\theta(\xx_0)] \bigg].
\end{align}
The detailed derivations are put in Appendix. We formally define the DCD-VE objective for training EBM as $\mathcal{L}_{DCD}^{(VE)}(\theta)$ in (\ref{eqn:dcd_ve}) with a small perturbation level $t$. For one-step $\mathcal{L}_{DCD}^{(VE)}(\theta)$, if data is low dimensional, the second order derivative is computationally tractable. However, for high dimensional data such as natural images, the second order derivative (the Laplacian term) can be efficiently estimated by the widely-used Hutchinson’s trace estimation techniques \citep{Hutchinson1989ASE,Chen2019ResidualFF,Grathwohl2019FFJORDFC,Song2019SlicedSM,song2021scorebased}.

\subsection{Train time-dependent EBM with DCD}
Inspired by recent success on score-based diffusion models \citep{gao2020learning,song2020improved,song2020score,song2021maximum}, learning a diffusion time-dependent EBM helps for better generative performance. In this section, we modify our DCD-VE for training time-dependent EBMs. Assuming $(\bm{F},\bm{G})$ denotes a pre-defined forward diffusion process \eqref{form:1} (as we use when defining DCD). Let $p_d^{(0)}$ denotes the data distribution, and $p_d^{(t)}$ denotes the $t$-time diffused data distribution initialized with $p_d^{(0)}$. A time-dependent EBM is a $f_\theta^{(t)}$ if a neural network that takes both $\xx$ and time $t$ to output the energy function of a point $\xx$ at diffusion time $t$. One can train $f_\theta^{(t)}$ to model the diffused data energy $\log p_d^{(t)}(\xx)$ at any time $t$. More precisely, at each training iteration, we randomly pick a timestamp $t\sim Unif([0, T])$, and apply DCD training at timestamp $t$ with a small diffusion perturbation $\delta$. In practice, if we discretize the time interval of a diffusion process to $\{t_i\}_{i=1,.., K}$, the perturbation $\delta$ can be chosen to be $\delta_i = t_i - t_{i-1}$ for different time $t_i$. Such a setting combines the DCD and diffusion process in a more natural way. We summarize the DCD training for time-dependent EBM in an Algorithm in the Appendix. 

\section{Experiments}
\subsection{Energy modeling of 2D distributions}
In this section, we validate our proposed DCD on 7 commonly used 2D synthetic datasets. This experiment shows that DCD is capable of learning challenging distributions such as the Checkerboard distribution whose distribution changes rapidly (as shown in the left part of Figure \ref{fig:toy_and_celeba}).

\paragraph{Experiment Setting.}
We use a 3-layer MLP with Gaussian Error Linear Unit (GELU) activations \citep{Hendrycks2016GaussianEL} and 300 hidden units for implementation of the EBM. We compare the DCD-VE with CD and Persistent Contrastive Divergence (PCD)\citep{tieleman2009using}, which is a well-known variant of CD. Since the CD training requires many iterations of inference of the EBM, we limit the times of score function evaluation to 10 times to make an equal comparison. We set the training batch size to be 1000 and PCD's replay buffer size to be 10 times the batch size. All models share the same architecture and the same training setting. We put detailed settings in Appendix.

\paragraph{Evaluation metric.} We compute the score-matching loss over the training data as the evaluation metric. The score matching loss is defined with 
\begin{align*}
    \operatorname{L}(\theta) \coloneqq \mathbb{E}_{\xx \sim p_d} \bigg[ \frac{1}{2}\|\nabla_{\xx} f_\theta(\xx)\|_2^2 + \Delta_{\xx} f_\theta(\xx) \bigg].
\end{align*}
So the smaller the SM loss is, the better the learning performance of the EBM.

\begin{table*}
\caption{Estimated SM loss of learned EBM.}
\label{tab:2}
\vskip 0.05in
\begin{center}
\begin{scriptsize}
\begin{tabular}{lccccccc}
\toprule
Dataset & Swissroll & Circles & Rings & Moons & 8 Gaussians & 2 Spirals & Checkerboard  \\
\midrule
\textbf{DCD-VE} & \textbf{-2398.81} & \textbf{-131.37}  & \textbf{-758.33} & \textbf{-200.67} & \textbf{-120.09} & \textbf{-470.92} & \textbf{-178.43} \\
CD & $+\infty$ & -130.03  & $+\infty$ & -195.72  & -117.29 & $+\infty$ & -67.22  \\
PCD & $+\infty$ & -108.54 & $+\infty$ & -193.76 & -97.59 & $+\infty$ & -124.27 \\
\bottomrule
\end{tabular}
\end{scriptsize}
\end{center}
\end{table*}

\begin{figure}
\centering
\subfigure[Comparison of CD, PCD and DCD-VE]{\includegraphics[width=0.6\linewidth]{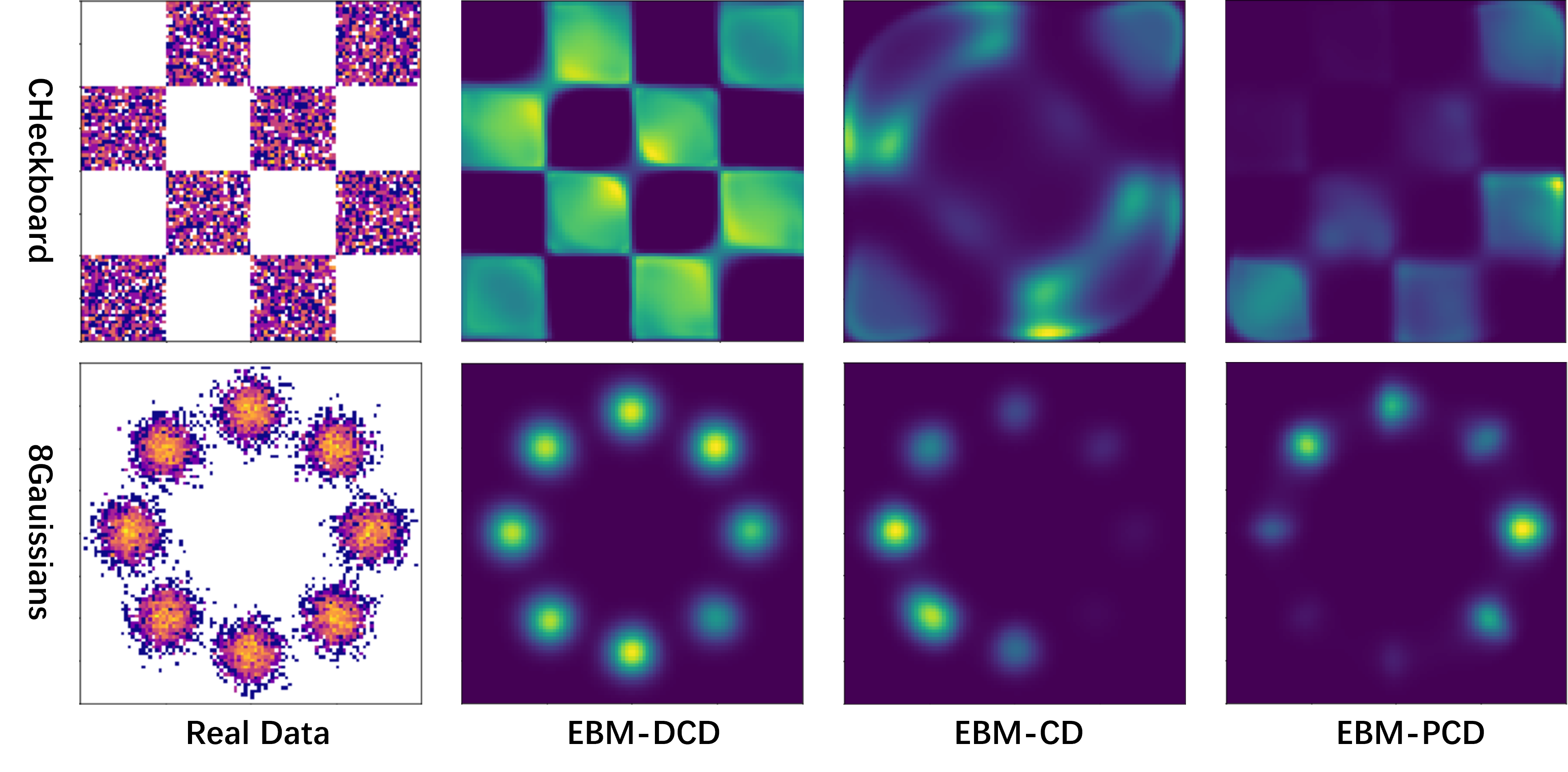}}
\centering
\subfigure[Generated CelebA $32$ samples from EBM.]{\includegraphics[width=0.3\linewidth]{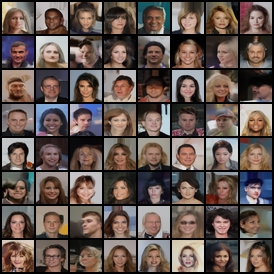}}
\caption{\textit{Left}: 2D examples when CD and PCD fails to learn a correct EBM but DCD-VE can learn successfully; \textit{Right}: Generated CelebA $32$ samples from EBM trained with DCD-VE.}
\label{fig:toy_and_celeba}
\end{figure}

\paragraph{Performance.}
We estimate the Score Matching (SM) loss (\citep{Song2019SlicedSM,Meng2020AutoregressiveSM}) on training data to evaluate the trained EBM. The smaller the SM loss, the better performance the EBM behaves. Table \ref{tab:2} shows the resulting SM losses for EBMs that are trained with DCD-VE, CD, and PCD. Since the SM loss is the training objective of SM-related training methods, we do not include them in the comparison. 
As is shown in Table \ref{tab:2}, DCD-VE outperforms CD and PCD on all datasets by a significant margin. Besides, CD and related methods do not converge on the more challenging Swiss roll, Rings, and 2Spirals dataset, while the DCD-VE can learn all data energy equally well. Figure \ref{fig:dcd_toy} demonstrates the learned energies on five datasets with DCD-VE.

\begin{wraptable}{r}{4.5cm}
\caption{CelebA.}\label{fig:celaba}
\small
\begin{tabular}{ll}
\toprule
\textbf{Models} & \textbf{FID} $\downarrow$ \\
\midrule
ABP \citep{HanLZW17} & 51.50  \\
ABP-SRI \citep{NijkampP0ZZW20} & 36.84 \\
VAE \citep{kingma2013auto} & 38.76 \\
Glow \citep{kingma2018glow} & 23.32 \\
DCGAN \citep{radford2015unsupervised} & 12.50 \\
EBM-FCE \citep{gao2020flow} & 12.21 \\
GEBM \citep{arbel2020generalized} & \; 5.21\\
CoopFlow(T=30) \citep{xie2022tale} & \; 6.44 \\
\midrule
EBM-DCD & \; 13.85 \\
\bottomrule
\end{tabular}
\end{wraptable}

\begin{figure}[h]
\centering
\includegraphics[width=0.8\linewidth]{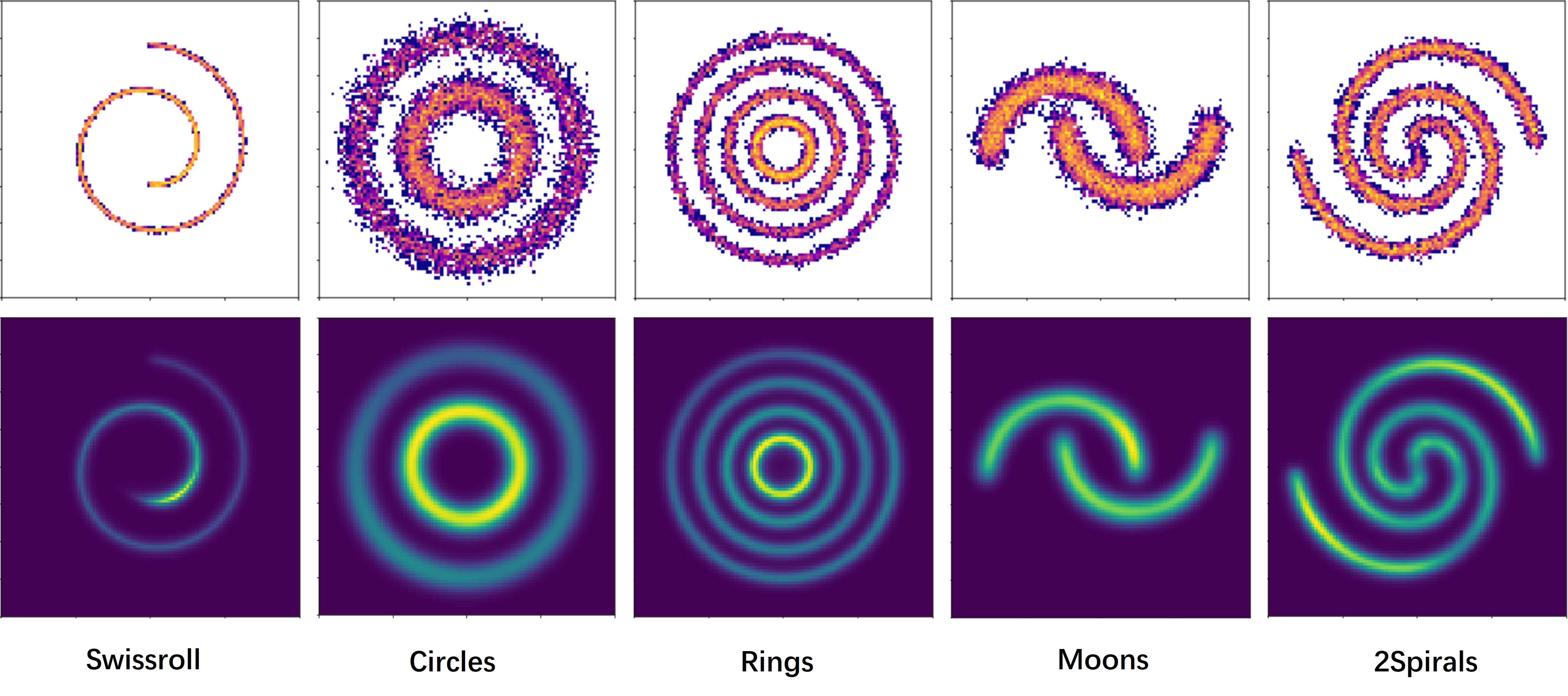}
\caption{Comparison of different training methods.}
\label{fig:dcd_toy}
\end{figure}

\begin{figure}[h]
\centering
\includegraphics[width=0.8\linewidth]{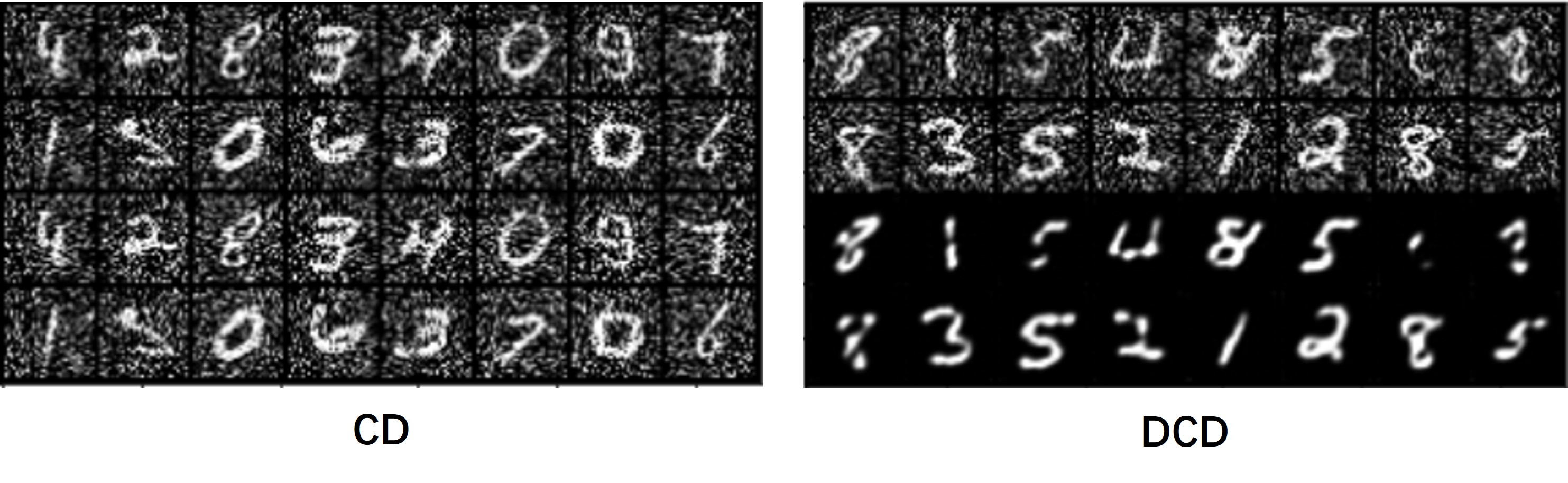}
\caption{The CD fails to denoise large added noise, while the DCD (VE) can denoise successfully.}
\label{fig:image_denoise}
\end{figure}

\subsection{Image denoising with EBM}
Image denoising is a common task to test explicit generative models \cite{Meng2020AutoregressiveSM}.

In this section, we validate the proposed DCD for training EBM on high-dimensional datasets and evaluate the image-denoising performance on four datasets, the MNIST, FashionMNIST, CIFAR10, and the SVHN datasets. 

\paragraph{Experiment Setting.}
We train EBM with DCD-VE and compare it with CD. We added the Gaussian noise with three strength levels on test images and evaluate the average root of the mean of the squared error (RMSE) of non-noised and denoised images. For the implementation of the EBM, we use the wide resnet\citep{Zagoruyko2016WideRN} model with GELU activations as our energy model. More details are put in Appendix.

Table 4 shows the denoising performance of Gaussian noise with different scales of the noise (low for 0.3, middle for 0.6, and high for 0.9). The DCD-VE performs consistently better than CD across different datasets and different noise strengths. We also surprisingly find that one advantage of the DCD is its impressive performance on large noise strength. As Figure \ref{fig:image_denoise} shows, for a high noise scale of 0.9, the EBM trained with CD fails to denoise successfully, while the EBM trained with DCD-VE still shows denoising ability.

\begin{table*}
\centering
\label{tab:exp_denoise}
\begin{scriptsize}
\caption{Average RMSE of clean and reconstructed input with Gaussian noise on datasets. (We set low, mid, and high-level noise as 0.3, 0.6, and 0.9.)}
\begin{center}
\begin{tabular}{l|cccccccccccc}
\toprule
\multicolumn{1}{c}{\multirow{2}{*}{Method}} 
& \multicolumn{3}{c}{MNIST}
& \multicolumn{3}{c}{FMNIST}                                 
& \multicolumn{3}{c}{CIFAFR10} 
& \multicolumn{3}{c}{SVHN}  \\
\multicolumn{1}{c}{} & low & mid & high & low & mid & high & low & mid      & high & low & mid & high    \\
\midrule
DCD & \textbf{0.165} & \textbf{0.194} & \textbf{0.303} & \textbf{0.170} & \textbf{0.217} & \textbf{0.497} & \textbf{0.129} & \textbf{0.193} &  \textbf{0.244} & \textbf{0.099} & \textbf{0.137} & \textbf{0.294} \\
CD & $0.194$ & $0.390$ & $1.099$ & $0.171$ & $0.2792$ & $0.872$ & $0.154$ & $0.317$ & $8.572$ & $0.124$ & $0.293$ & $6.938$ \\
\bottomrule
\end{tabular}
\end{center}
\end{scriptsize}
\end{table*}

\subsection{Image generation with time-dependent EBM}
\paragraph{Experiment Setting.} We use DCD to train EBMs for image generation on the CelebA dataset of a resolution of $32\times 32$. We use a time-dependent neural network with residual network architecture \citep{He2016DeepRL} as the implementation of the EBM. 
We use the VE diffusion with the diffusion coefficient $g(t)=t$ as the forward diffusion, which is the same as \citet{Karras2022ElucidatingTD}. We train the time-dependent energy-based model on the CelebA dataset which is downsampled to have a resolution of 32x32. We evaluate the Frechet Inception Score (FID) \citep{heusel2017gans} as a metric of generation performance. 

\paragraph{Performance.} 
Table \ref{fig:celaba} shows the performance of our trained EBMs for a generation. It shows that the DCD (VE) is capable of handling complex image datasets. It demonstrates that the proposed DCD is capable of training EBMs with comparable performance to DCGAN \cite{radford2015unsupervised} and other EBMs (i.e. EBM with FCE \citep{gao2020flow}), and superior performance to normalizing flow models and VAE. However, the performance is worse than EBM which requires more advanced tricks such as cooperating with flow models (CoopFlow \citep{Xie2022ATO}) and cooperating with GAN models (GEBM \citep{arbel2020generalized}). The right-hand side of Figure \ref{fig:toy_and_celeba} shows some generated samples from our trained EBM. In summary, the proposed DCD-VE is able to train time-dependent EBM with comparable generative performance as existing training methods. 

\section{Limitations and Future Works}
In this paper, we propose a novel family of probability divergences, the \emph{diffusion contrastive divergence} family. The DCD provides a special view that unifies the contrastive divergence as a special instance of the DCD. It also spurs new divergences for training EBM which overcomes two major drawbacks of the contrastive divergence. We also establish the connection of the proposed DCDs with existing recovery likelihood and the KL-contraction divergences. We validate the efficiency and superior performance of our proposed DCDs on several benchmark EBM tasks such as 2D energy modeling, image denoising, and image generation. 

However, the DCD also has its limitations. First, the calculation of DCD requires the computation of a higher-order derivative of the energy function, meaning that the energy-based model should be at least twice differentiable. Second, the long-time DCD requires the calculation of the evolved energy function. Such evolution is not easy to compute for the general diffusion process. We plan to leave the research of the long-time energy evolution of DCD in our further work.

\newpage

\bibliography{neurips_2023}
\bibliographystyle{IEEEtranN}


\newpage

\appendix
\section{Technical details}
\subsection{Proof of Theorem \ref{thm:1} (Section 3.1)}\label{app:prove_of_integral_thm}

Before we prove Theorem \ref{thm:1}, we give two lemmas to simplify the proof.
\begin{lemma}\label{lem:laplacian}
Assume function $p$ is positive and twice differentiable, then the following identity holds
\[
\Delta p(\xx) = p(\xx) \|\nabla_{\xx} \log p(\xx)\|_2^2 + p(\xx)\Delta \log p(\xx),
\]
where
\begin{align*}
    &\nabla_{\xx} \log p(\xx) = \sum_{i=1}^D \frac{\partial p(\xx)}{\partial {x}_i}, \ \Delta \log p(\xx) = \sum_{i=1}^D \frac{\partial^2 \log p(\xx)}{\partial {x}_i^2}.
\end{align*}
Here $x_i$ represents the $i$-th covariate of vector $\xx$ and $D$ is the data dimension.
\end{lemma}
\begin{proof}
With a slight abuse of notation, we write $\langle \nabla_{\xx}, \bm{f}(\xx) \rangle \coloneqq \sum_{i=1}^D \partial \bm{f}_i(\xx)/\partial {x}_i$. Since $p$ is twice differentiable, we have 
\begin{align*}
    \Delta p(\xx) &= \langle \nabla_{\xx}, \nabla_{\xx} p(\xx)\rangle \\
    &= \langle \nabla_{\xx}, p(\xx)\nabla_{\xx} \log p(\xx)\rangle =\langle \nabla_{\xx} p(\xx), \nabla_{\xx} \log p(\xx)\rangle + p(\xx)\langle \nabla_{\xx}, \nabla_{\xx}\log p(\xx)\rangle\\
    &= p(\xx)\|\nabla_{\xx} \log p(\xx)\|_2^2 + p(\xx)\Delta \log p(\xx)
\end{align*}
\end{proof}

\begin{lemma}\label{lem:stein}
Assume function $p(.)$ is a positive and twice differentiable probability density. Then the following identity holds
\begin{align*}
    \mathbb{E}_{p(\xx)}\log p(\xx) \|\nabla_{\xx} \log p(\xx)\|_2^2 = -\mathbb{E}_{p(\xx)} \big[ \|\nabla_{\xx} \log p(\xx)\|_2^2 + \log p(\xx) \Delta \log p(\xx) \big].
\end{align*}
\end{lemma}

\begin{proof}
Notice that
\begin{align*}
    \mathbb{E}_{p(\xx)}\log p(\xx) \|\nabla_{\xx} \log p(\xx)\|_2^2 &= \mathbb{E}_{p(\xx)} \langle \log p(\xx) \nabla_{\xx} \log p(\xx), \nabla_{\xx}\log p(\xx)\rangle \\
    &= \mathbb{E}_{p(\xx)}-\langle \nabla_{\xx}, \log p(\xx) \nabla_{\xx}\log p(\xx) \rangle.
\end{align*}
The above equality holds because of Stein's identity \citep{stein1981estimation}, i.e., 
\[
\mathbb{E}_{p(\xx)} \langle \bm{f}(\xx),\nabla_{\xx} \log p(\xx) \rangle = -\mathbb{E}_{p(\xx)}\langle \nabla_{\xx}, \bm{f}(\xx)\rangle
\] 
for vector value function $\bm{f}$ which lies in Stein class of $p$ \footnote{A vector-value function $\bm{f}$ lies in Stein class of distribution $p$ means three conditions hold: 
\begin{itemize}
    \item $\bm{f}$ is 2nd-order smooth;
    \item both $\|\bm{f}\|_2^2$  and $\|\nabla_{\xx} \bm{f}^T\|_F^2$ is integrable w.r.t. $p$. The notation $\|.\|_F$ represents the Frobenius norm.
    \item $p(\xx) \|\nabla_{\xx} \bm{f}^T(\xx)\|_F \to 0$ when $\|\xx\|_2 \to \partial support(p)$
\end{itemize}
}. 
Thus the proof is finished with
\begin{align*}
    \mathbb{E}_{p(\xx)}\log p(\xx) \|\nabla_{\xx} \log p(\xx)\|_2^2= -\mathbb{E}_{p(\xx)}\bigg[ \|\nabla_{\xx} \log p(\xx)\|_2^2 + \log p(\xx) \Delta \log p(\xx)\bigg]
\end{align*}
\end{proof}
We give the proof for Theorem \ref{thm:1} with the above two lemmas \ref{lem:laplacian} and \ref{lem:stein}.
\begin{proof}
Recall that the two distributions $p,q$ evolve along a general Ito's diffusion process $$d\xx_t = \bm{F}(\xx_t,t)\mathrm{d}t + \bm{G}(t)\mathrm{d} \bm{w}_t.$$
Here $\bm{F}(\xx,t)$ is a vector value function, and $\bm{G}(t)$ is a scalar function of $t$.
Note that $p_0=p,q_0=q$. We denote $p^{(F,G,t)}_t,q^{(F,G,t)}_t$ as $p_t,q_t$ for short. The KL divergence between $p_t,q_t$ is defined as 
$$\mathcal{D}_{KL}(p_t,q_t) = \mathbb{E}_{p_t}\log \frac{p_t(\xx)}{q_t(\xx)}= \int p_t\log \frac{p_t}{q_t}\diff \xx.$$
We declare all integrals are w.r.t. $\xx$ and omit the $\diff \xx$ in integral formulas for simplification. The change rate of KL divergence is
\begin{align}
    &\frac{\diff}{\diff t}\mathcal{D}_{KL}(p_t,q_t)\frac{\diff}{\diff t}\int p_t(\xx)\log \frac{p_t(\xx)}{q_t(\xx)}\nonumber\\
    &= \int \frac{\diff p_t}{\diff t}\log p_t - \int \frac{\diff p_t}{\diff t}\log q_t + \int \frac{\diff p_t}{\diff t} - \int  \frac{p_t}{q_t}\frac{\diff q_t}{\diff t}\label{eq:change_rate}\\
    &:= A + B + C + D.\nonumber
\end{align}
The third term 
$$C = \int \frac{\diff p_t}{\diff t} = \frac{\diff }{\diff t}\int p_t = \frac{\diff}{\diff t}1 = 0.$$ 
Hence the above equation remains 3 terms. 
By the Fokker-Planck equation \eqref{equ:fp_equation}, the evolved density follows
\begin{align*}
    \frac{\diff p_t}{\diff t} &= -\langle \nabla_{\xx}, p_t \bm{F} \rangle + \frac{1}{2}\bm{G}^2(t)\Delta p_t\\
    &= -p_t \langle \nabla_{\xx} \log p_t, \bm{F}\rangle - p_t \langle \nabla_{\xx}, \bm{F}\rangle + \frac{1}{2}\bm{G}^2(t) p_t \|\nabla_{\xx} \log p_t\|_2^2 + \frac{1}{2}\bm{G}^2(t) p_t \Delta \log p_t.
\end{align*}
Substitute the above equation to equation (\ref{eq:change_rate}), the first term $A$ becomes 
\begin{align}\label{eq:1st_term}
    &\int \frac{\diff p_t}{\diff t}\log p_t\\
    &= \int p_t \bigg[ \frac{1}{2}\bm{G}^2(t)\|\nabla_{\xx} \log p_t\|_2^2 + \frac{1}{2}\bm{G}^2(t)\Delta \log p_t -\bm{F}(x,t)^T \nabla_{\xx} \log p_t - \langle \nabla_{\xx}, \bm{F}(\xx,t)\rangle \bigg]\log p_t \nonumber\\
    &=\mathbb{E}_{p_t} \bigg[ \frac{1}{2}\bm{G}^2(t)\log p_t \|\nabla_{\xx} \log p_t\|_2^2 + \frac{1}{2}\bm{G}^2(t)\log p_t \Delta \log p_t \nonumber \\
    & - \log p_t\langle \bm{F}(\xx,t),\nabla \log p_t \rangle -\log p_t \langle \nabla_{\xx}, \bm{F}(\xx,t)\rangle  \bigg] \nonumber\\
    &=\mathbb{E}_{p_t} \bigg[ -\frac{1}{2}\bm{G}^2(t)\big[ \|\nabla_{\xx} \log p_t\|_2^2 + \log p_t \Delta \log p_t \big] + \frac{1}{2}\bm{G}^2(t)\log p_t \Delta \log p_t \nonumber \\
    & - \log p_t\langle \bm{F}(x,t),\nabla \log p_t \rangle -\log p_t \langle \nabla_{\xx}, \bm{F}(\xx,t)\rangle  \bigg]. \nonumber
\end{align}
By Stein's identity, 
\begin{align*}
    &\mathbb{E}_{p_t} \log p_t \langle \bm{G},\nabla_{\xx} \log p_t \rangle= \mathbb{E}_{p_t} \langle (\log p_t) \bm{F},\nabla_{\xx} \log p_t \rangle\\
    &= -\mathbb{E}_{p_t} \langle \nabla_{\xx}, \log p_t \bm{F} \rangle= -\mathbb{E}_{p_t} \bigg[ \langle \nabla_{\xx}\log p_t, \bm{F} \rangle + \log p_t\langle\nabla_{\xx},\bm{F} \rangle \bigg].
\end{align*}
This term (\ref{eq:1st_term}) becomes
\begin{align*}
    &\mathbb{E}_{p_t} \bigg[ -\frac{1}{2}\bm{G}^2(t) \|\nabla_{\xx} \log p_t\|_2^2 - \log p_t\langle \bm{F},\nabla_{\xx} \log p_t \rangle -\log p_t \langle \nabla_{\xx}, \bm{F}\rangle  \bigg]\\
    &=\mathbb{E}_{p_t} \bigg[ -\frac{1}{2}\bm{G}^2(t) \|\nabla_{\xx} \log p_t\|_2^2 + \big[ \langle \bm{F}, \nabla_{\xx} \log p_t\rangle + \log p_t \langle \nabla_{\xx}, \bm{F}\rangle \big] -\log p_t \langle \nabla_{\xx}, \bm{F}\rangle  \bigg]\\
    &=\mathbb{E}_{p_t} \bigg[ -\frac{1}{2}\bm{G}^2(t) \|\nabla_{\xx} \log p_t\|_2^2 + \langle \bm{F}, \nabla_{\xx} \log p_t\rangle  \bigg].
\end{align*}
Next, we calculate the second term $B$ with a similar argument
\begin{align*}
    &B = \int -\frac{\diff p_t}{\diff t}\log q_t\\
    &= -\mathbb{E}_{p_t} \log q_t \bigg[\frac{1}{2}\bm{G}^2(t)\|\nabla_{\xx} \log p_t\|_2^2 + \frac{1}{2}\bm{G}^2(t)\Delta \log p_t - \langle \bm{F},\nabla_{\xx} \log p_t \rangle - \langle\nabla_{\xx}, \bm{F} \rangle \bigg]\\
    &=-\mathbb{E}_{p_t} \bigg[ -\frac{1}{2}\bm{G}^2(t)\langle \nabla_{\xx} \log p_t, \nabla_{\xx} \log q_t\rangle + \langle \bm{F}, \nabla_{\xx} \log q_t\rangle \bigg].
\end{align*}
The fourth term $D$ writes
\begin{align*}
    D& = \int -\frac{p_t}{q_t} \frac{\diff q_t}{\diff t}\\
    &= -\int\frac{p_t}{q_t} q_t\bigg[ \frac{1}{2}\bm{G}^2(t)\|\nabla_{\xx} \log q_t\|_2^2 + \frac{1}{2}\bm{G}^2(t)\Delta \log q_t - \langle \bm{F}, \nabla_{\xx} \log q_t\rangle - \langle\nabla_{\xx}, \bm{F} \rangle \bigg]\\
    &=-\mathbb{E}_{p_t}\bigg[ \frac{1}{2}\bm{G}^2(t)\|\nabla_{\xx} \log q_t\|_2^2 + \frac{1}{2}\bm{G}^2(t)\Delta \log q_t - \langle \bm{F}, \nabla_{\xx} \log q_t\rangle - \langle\nabla_{\xx}, \bm{F} \rangle \bigg].
\end{align*}
With Stein's identity, 
$$\mathbb{E}_{p_t}\langle \nabla_{\xx} \log q_t, \nabla_{\xx} \log p_t\rangle = -\mathbb{E}_{p_t} \langle\nabla_{\xx}, \nabla_{\xx} \log q_t \rangle = -\mathbb{E}_{p_t} \Delta \log q_t.$$
Substitute the above equality to the fourth term, we have
\begin{align*}
    D =-\mathbb{E}_{p_t} \bigg[ \frac{1}{2}\bm{G}^2(t)\|\nabla_{\xx} \log q_t\|_2^2 - \frac{1}{2}\bm{G}^2(t)\langle\nabla_{\xx} \log q_t, \nabla_{\xx} \log p_t \rangle -\langle \bm{F}, \nabla_{\xx} \log q_t\rangle + \langle \bm{F},\nabla_{\xx} \log p_t \rangle \bigg].
\end{align*}
Combine all three terms, we have
\begin{align}
    &\frac{d}{\diff t}\mathcal{D}_{KL}(p_t,q_t) = \int \frac{\diff p_t}{\diff t}\log p_t - \int \frac{\diff p_t}{\diff t}\log q_t - \int \frac{p_t}{q_t} \frac{\diff q_t}{\diff t}\nonumber \\
    &= -\mathbb{E}_{p_t}\bigg[ \frac{1}{2}\bm{G}^2(t)\|\nabla_{\xx} \log p_t\|_2^2 + \frac{1}{2}\bm{G}^2(t)\|\nabla_{\xx} \log q_t\|_2^2 - \bm{G}^2(t)\langle \nabla_{\xx} \log p_t, \nabla_{\xx} \log q_t\rangle \bigg] \nonumber\\
    \label{equ:dkl_dt}
    &=-\frac{1}{2}\mathbb{E}_{p_t} \bm{G}^2(t)\|\nabla_{\xx} \log p_t(x) - \nabla_{\xx} \log q_t(x)\|_2^2
\end{align}
So the integral representation writes
\begin{align*}
    &\mathcal{D}_{KL}(p_T,q_T) - \mathcal{D}_{KL}(p_0,q_0) = \int_0^T \frac{d}{\diff t}\mathcal{D}_{KL}(p_t,q_t)\diff t\\
    &= -\int_0^T \frac{1}{2}\mathbb{E}_{p_t} \bm{G}^2(t)\|\nabla_{\xx} \log p_t(\xx) - \nabla_{\xx} \log q_t(\xx)\|_2^2\diff t.
\end{align*}
\end{proof}

\subsection{Proof of Langevin dynamic's stationary property}\label{app:ld_stationary}
The stationary property states that $p_\theta$ is stationary under EBM-induced Langevin dynamics. 
\begin{proof}
Notice that the evolution of a probability under EBM Langevin dynamics \ref{equ:ld} is governed by the Fokker-Planck equation \eqref{equ:fp_equation}
\[
\frac{\diff}{\diff t}p(\xx,t) = -\langle\nabla_{\xx}, \frac{1}{2}\nabla_{\xx} \log p_\theta(\xx)p(\xx,t) \rangle + \frac{1}{2}\Delta_{\xx} p(\xx,t).
\]
Since $\Delta p(\xx,t) = \langle\nabla_{\xx}, \nabla_{\xx} p(\xx,t) \rangle$, we have 
\begin{align*}
    \Delta_{\xx} p(\xx,t) = \langle\nabla_{\xx}, \nabla_{\xx} p(\xx,t) \rangle =\langle \nabla_{\xx}, p(\xx,t)\nabla_{\xx} \log p(\xx,t)\rangle
\end{align*}
Combining the above, we have the simplified Fokker-Planck equation
\begin{align*}
    \frac{\diff}{\diff t}p(\xx,t) = \frac{1}{2}\langle\nabla_{\xx}, \frac{1}{2}\big[ \nabla_{\xx} \log p(\xx,t)-\nabla_{\xx} \log p_\theta(\xx)\big]p(\xx,t) \rangle
\end{align*}
Substitute $p(\xx,t) = p_\theta(\xx)$, we have
\[
 \frac{\diff}{\diff t}p(\xx,t) = 0.
\]
So $p(\xx,t) = p_\theta(\xx)$ is stationary under $p_\theta$ induced Langevin dynamics.
\end{proof}

\subsection{Non-negativity of CD (Theorem \ref{thm:1})}
Recall the definition of CD \eqref{eqn:cd_0},
\begin{align*}
    \mathcal{D}_{CD}(p_d,p_\theta) = \mathcal{D}_{KL}(p_d, p_\theta) - \mathcal{D}_{KL}(p_{d,\theta}^{(T)}, p_\theta),
\end{align*}
The non-negativity of CD in fact comes as a corollary of Theorem \ref{thm:1} as we have proved in \ref{app:prove_of_integral_thm}.
\begin{proof}
Recall the definition of CD,
\[
\mathcal{D}_{CD}(p_d,p_\theta) = \mathcal{D}_{KL}(p_d, p_\theta) - \mathcal{D}_{KL}(p_d^{(T)}(\theta), p_\theta).
\]
Here the $p_{d}^{(T)}(\theta)$ denote the $T$ time evolved EBM distribution under EBM Langevin dynamcis
\[ 
d\xx_t = \frac{1}{2}\nabla_{\xx_t}\log p_\theta(\xx_t) \diff t + \diff \bm{w}_t.
\]
Recall that $p_\theta(\xx) = p_\theta^{(T)}(\xx)$ as we show in \ref{app:ld_stationary}, then in definition of CD and CD equals to 
\[
\mathcal{D}_{CD}(p_d,p_\theta) = \mathcal{D}_{KL}(p_d, p_\theta) - \mathcal{D}_{KL}(p_d^{(T)}(\theta), p_\theta^{(T)}).
\]
By Theorem \ref{thm:1}, 
\begin{align*}
    \mathcal{D}_{CD}(p_d,p_\theta) = \frac{1}{2}\int_{0}^T \mathbb{E}_{\xx_t \sim p_d^{(t)}(\xx_t)}\|\nabla_{\xx_t} \log p_\theta^{(t)}(\xx_t)- \nabla_{\xx_t} \log q^{(t)}(\xx_t)\|^2_2 \diff t\geq 0
\end{align*}
\end{proof}

\subsection{Non-negaligibility of the extra term of CD (Equation 5)}
Recall the gradient formula of CD \eqref{eq:cd_grad}.
\begin{align*}
    &\frac{\partial}{\partial \theta}\mathcal{D}_{CD}(p_d,p_\theta) = \mathbb{E}_{ p_{d,\theta}^{(T)}}\big[\frac{\partial}{\partial \theta}f_\theta(\xx) \big] -\mathbb{E}_{ p_d}\big[\frac{\partial}{\partial \theta} f_\theta(\xx)\big] - \mathbb{E}_{p_{d,\theta}^{(T)}}\bigg[\log p_\theta(\xx) \frac{\partial}{\partial \theta}\log p_{d,\theta}^{(T)}(\xx) \bigg].
\end{align*}
The third term is 
\begin{align}\label{term:third_term}
    (3) = - \mathbb{E}_{p_{d,\theta}^{(T)}}\bigg[\log p_\theta(\xx) \frac{\partial}{\partial \theta}\log p_{d,\theta}^{(T)}(\xx) \bigg].
\end{align}
For ease of expression, we may omit the notation $\diff \xx$ in the integral. If $p_{d,\theta}^{T}(\xx) \to p_\theta(\xx)$ as we assumed, the term \ref{term:third_term} turns to 
\begin{align}
    (3) &= - \mathbb{E}_{p_\theta}\bigg[\log p_\theta(\xx) \frac{\partial}{\partial \theta}\log p_\theta(\xx) \bigg] \nonumber\\
    &= -\int p_\theta(\xx) \log p_\theta(\xx) \frac{1}{p_\theta(\xx)}\frac{\partial}{\partial\theta}p_\theta(\xx) \diff \xx \nonumber \\
    &= -\int \log p_\theta(\xx) \frac{\partial}{\partial\theta} p_\theta(\xx)\diff \xx \nonumber\\
    &= - \frac{\partial}{\partial\theta} \int \log p_\theta(\xx) p_\theta(\xx)\diff \xx + \int p_\theta(x)\frac{\partial}{\partial\theta} \log p_\theta(\xx) \nonumber \\
    & = -\frac{\partial}{\partial\theta}\mathbb{E}_{p_\theta} \log p_\theta(\xx) + \int \frac{\partial}{\partial\theta} p_\theta(\xx)\nonumber \\
    \label{term:3}
    & = -\frac{\partial}{\partial\theta} \mathbb{E}_{p_\theta} \log p_\theta(\xx) + \frac{\partial}{\partial\theta}\int p_\theta(\xx) \\
    & = -\frac{\partial}{\partial\theta} \mathbb{E}_{p_\theta} \log p_\theta(\xx) + \frac{\partial}{\partial\theta} \mathbf{1} \nonumber \\
    & = -\frac{\partial}{\partial\theta} \mathbb{E}_{p_\theta} \log p_\theta(\xx) \nonumber
\end{align}
The equality \ref{term:3} holds if $p_\theta(x)$ satisfies the conditions. Now if the density function is satisfied the condition that (1). $p_\theta(x)$ is Lebesgue integrable for $\xx$ with each $\theta$; (2). For almost all $\xx \in \mathbf{R}^D$, the partial derivative $\partial p_\theta(\xx)/\partial\theta$ exists for all $\theta\in \Theta$. (3) there exists an integrable function $g(.): \mathbf{R}^D \to \mathbf{R}$, such that $p_\theta(\xx) \leq g(\xx)$ for all $\xx$ in its domain. Then the derivative w.r.t $\theta$ can be exchanged with the integral over $\xx$, i.e. 
\begin{align*}
\int \frac{\partial}{\partial\theta}p_\theta(\xx)\diff \xx = \frac{\partial}{\partial \theta} \int p_\theta(\xx)\diff \xx.
\end{align*}

\subsection{Tractable form of DCD with parameter-free diffusion}
Thus the DCD under the above SDE has the form
\begin{align*} &\mathcal{D}_{DCD}^{(\bm{F},\bm{G},T)}(p_d, p_\theta) \\
    &= \mathbb{E}_{\xx_0\sim p_d}\big[\log p_d(\xx_0) - f_\theta(\xx_0) \big] - \mathbb{E}_{\xx_0\sim p_d,\atop \xx_t\sim p(\xx_t|\xx_0)}\big[\log p_d^{(\bm{F},\bm{G},T)}(\xx_t) - f_\theta^{(\bm{F},\bm{G},T)}(\xx_t) \big].
\end{align*}
Here $f_\theta^{(\bm{F},\bm{G},T)}$ are time $T$ marginal energy under diffusion process (\eqref{form:1}). The term $\log p_d(\xx)$ and $\log p_d^{(\bm{F},\bm{G},T)}$ are independent of parameter $\theta$ since the diffusion process is parameter-free. As a result, we can drop them when using gradient-based optimization algorithms. Thus, we have the final tractable learning objective based on DCD as
\begin{align*}
    \mathcal{L}_{DCD}(\theta) =& \mathbb{E}_{\xx_0\sim p_d, \xx_t\sim p(\xx_t|\xx_0)}\big[f_\theta^{(\bm{F},\bm{G},T)}(\xx_t)\big] - \mathbb{E}_{\xx_0\sim p_d}\big[f_\theta(\xx_0) \big].
\end{align*}

\subsection{More backgrounds on VE diffusion}
\paragraph{The VE Diffusion.}
Recall the VE diffusion \eqref{eq:ve_forward},
\begin{align*}
    \diff \xx_t = g(t)\diff \bm{w}_t,
\end{align*}
is also favored for its easy-to-simulate property. The marginal transition of VE writes
\begin{align}
    p(\xx_t|\xx_0) = \mathcal{N}(\xx_0, \sigma(t)\mathbf{I}).
\end{align}
$\sigma(t) = \int_0^t g(s)\diff s$. Similar to the VP diffusion, marginal samples of the VE diffusion are also cheap to obtain and parameter-free.

\subsection{Proof of Proposition 1 (Section 3.1)}
Recall the Proposition \ref{sec:energy_evolve}.
\begin{proposition*}
Assume $p_\theta^{(0)}(\xx) = e^{f_\theta(\xx)}/Z_\theta$ where $Z_\theta$ is a parameter-dependent normalizing constant. Assume $p_\theta^{(t)}$ denotes the evolved density along a diffusion process \eqref{form:1}, then for any fixed $\xx$, the energy value $p_\theta^{(t)}(\xx)$ evolves according to a PDE
\[
d\log p_\theta^{(t)}(\xx)/\diff t=\mathcal{O}(\nabla_{\xx} \log p_\theta^{(t)}),
\]
where $\mathcal{O}(\nabla_{\xx} \log p_\theta^{(t)})$ is the following operator which is independent of the normalizing constant, 
\begin{align*}
    &\langle \bm{G}^2(t)\nabla_{\xx} \log p_\theta^{(t)}(\xx)/2 - \bm{F}(\xx,t), \nabla_{\xx} \log p_\theta^{(t)}(\xx) \rangle + \langle \nabla, \bm{G}^2(t)\nabla_{\xx} \log p_\theta^{(t)}(\xx)/2 - \bm{F}(\xx,t) \rangle.
\end{align*}
\end{proposition*}

\begin{proof}
    Following the Fokker-Planck equation \ref{equ:fp_equation}, the density $p_\theta^{(t)}(\xx) = e^{f_\theta^{(t)}(\xx)}/Z_\theta$ evolves with equation:
\begin{align*}
    &\frac{\diff}{\diff t}p_\theta^{(t)}(\xx) = \\
    & =\langle \bm{G}^2(t)\nabla_{\xx} \log p_\theta^{(t)}(\xx)/2 - \bm{F}(\xx,t), \nabla_{\xx} \log p_\theta^{(t)}(\xx) \rangle + \langle \nabla, \bm{G}^2(t)\nabla_{\xx} \log p_\theta^{(t)}(\xx)/2 - \bm{F}(\xx,t) \rangle.
\end{align*}
Denote $f_\theta^{(t)}(\xx) = \log p_\theta^{(t)}(\xx) + \log Z^{(t)}_\theta$. Then we have $\nabla_{\xx} \log p_\theta^{(t)}(\xx) = \nabla_{\xx} f_\theta^{(t)}(\xx)$. The evolution of $f_\theta^{(t)}$ thus follows a partial differential equation (Fokker-Planck equation), i.e.,
\begin{align*}
    &\diff f_\theta^{(t)}(\xx)/\diff t = \langle \bm{G}^2(t)\nabla_{\xx} f_\theta^{(t)}(\xx)/2 - \bm{F}(\xx,t), \nabla_{\xx} f_\theta^{(t)}(\xx) \rangle + \langle \nabla_{\xx}, \bm{G}^2(t)\nabla_{\xx} f_\theta^{(t)}(\xx)/2 - \bm{F}(x,t) \rangle\\
    &=\mathcal{O}(f_\theta^{(t)},\xx).
\end{align*}
In the above equation,
$$\mathcal{O}(f) = \bm{G}^2\| \nabla_{\xx} f \|^2 - \langle \bm{F}(.), \nabla_{\xx} \bm{F}\rangle + \bm{G}^2/2 \Delta f - \langle \nabla_{\xx}, \bm{F}(.) \rangle.$$
Thus the $T$ time energy function equals $$f_\theta^{(T)}(\xx) = f_\theta(\xx) + \int_{t=0}^T \mathcal{O}(f_\theta^{(t)},\xx)\diff t,$$
where $f_\theta^{(t)}$ is the solution of the above energy diffusion ODE. We can derive the change of normalizing constant with the following argument. By writing
$$f_\theta^{(t+\diff t)} = \mathcal{O}(f_\theta^{(t)})\diff t + f_\theta^{(t)},$$
we have 
\[
\exp(f_\theta^{(t+\diff t)}) =\exp(f_\theta^{(t)})\exp(\mathcal{O}(f_\theta^{(t)})\diff t)=\exp(f_\theta^{(t)})(1+\mathcal{O}(f_\theta^{(t)}) + o(\diff t^2)).
\]
Taking integral w.r.t $x$ on both sides, we have
\begin{align}
    Z_\theta^{(t+\diff t)} &=\int \exp(f_\theta^{(t+\diff t)})\diff \xx \nonumber\\
    &= \int \exp(f_\theta^{(t)})\bigg[1+\mathcal{O}(f_\theta^{(t)}) + o(\diff t^2) \bigg]dx \nonumber\\
    &= Z_\theta^{(t)}\bigg[1+ \int \frac{\exp(f_\theta^{(t)})(\xx)}{Z_\theta^{(t)}}\mathcal{O}(f_\theta^{(t)})(\xx)\diff \xx\bigg] + o(\diff t^2) \nonumber\\
    &= Z_\theta^{(t)}\bigg[1+ \int \mathbb{E}_{p_\theta^{(t)}}\mathcal{O}(f_\theta^{(t)})(\xx)\diff \xx\bigg] + o(\diff t^2)\nonumber\\
    &= Z_\theta^{(t)}\bigg[1+ \mathbb{E}_{p_\theta^{(t)}}\mathcal{O}(f_\theta^{(t)})(\xx)\bigg] + o(\diff t^2).\label{eq:nc_ode}
\end{align}

Note that 
\begin{align*}
    &\mathbb{E}_{p_\theta^{(t)}}\mathcal{O}(f_\theta^{(t)})(\xx)=\mathbb{E}_{p_\theta^{(t)}}\mathcal{O}(\log p_\theta^{(t)})(\xx)\\
    &=\mathbb{E}_{p_\theta^{(t)}}\bigg[\langle \bm{G}^2(t)\nabla_{\xx} \log p_\theta^{(t)}(\xx)/2 - \bm{F}(\xx,t), \nabla_{\xx} \log p_\theta^{(t)}(\xx) \rangle + \langle \nabla_{\xx}, \bm{G}^2(t)\nabla_{\xx} \log p_\theta^{(t)}(\xx)/2 - \bm{F}(\xx,t) \rangle \bigg]\\
    &=\mathbb{E}_{p_\theta^{(t)}}\bigg[\frac{1}{2}\bm{G}^2(t)\|\nabla_{\xx} \log p_\theta^{(t)}(\xx)\|^2_2 + \bm{G}^2(t)\nabla_{\xx} \log p_\theta^{(t)}(\xx) + \bm{F}^T(\xx,t)\nabla_{\xx} \log p_\theta^{(t)}(\xx) + \nabla_{\xx} \bm{F}(\xx,t) \bigg].
\end{align*}
Through Stein's identity, we have 
$$\mathbb{E}_{p_\theta^{(t)}}\bigg[\|\nabla_{\xx} \log p_\theta^{(t)}(\xx)\|_2^2 + \Delta_{\xx} \log p_\theta^{(t)}(\xx) \bigg] = 0,$$
$$\mathbb{E}_{p_\theta^{(t)}}\bigg[ \bm{F}(\xx,t)^T\nabla_{\xx} \log p_\theta^{(t)}(\xx) + \nabla_{\xx} \mathbf{F}(\xx,t)\bigg]=0.$$
Thus we have 
\begin{align}
    \mathbb{E}_{p_\theta^{(t)}}\mathcal{O}(f_\theta^{(t)})(\xx)=0 \label{eq:equal0}
\end{align}
Substituting equation (\ref{eq:equal0}) into (\ref{eq:nc_ode}), we have
\begin{align*}
    Z_\theta^{(t+\diff t)} = Z_\theta^{(t)} + o(\diff t^2).
\end{align*}
Thus 
\begin{align*}
    \frac{\diff}{\diff t}Z_\theta^{(t)} = 0.
\end{align*}
The normalizing constant remains unchanged. Since the operator $\mathcal{O}$ only depends on the $\nabla_{\xx} f$ term, the normalizing constant does not influence the energy evolution. Then we have $\log p_\theta^{(t)}(\xx) = f_\theta^{(t)}(\xx) - \log Z_\theta$. So the normalizing constant $\log z_\theta$ can be abstracted in the loss function as
\begin{align*}
    &\mathcal{L}_{DCD}(\theta) = \mathbb{E}_{\xx_T\sim p_T(\xx_T)}\big[\log p_\theta^{(T)}(\xx_T) \big] - \mathbb{E}_{\xx_0\sim p_d}\big[\log p_\theta^{(0)}(\xx_0) \big]\\
    &= \mathbb{E}_{\xx_T\sim p_T(\xx_T)}\big[ f_\theta^{(T)}(\xx_T) -\log Z_\theta \big] - \mathbb{E}_{\xx_0\sim p_d}\big[\log p_\theta^{(0)}(\xx_0) - \log Z_\theta \big]\\
    &= \mathbb{E}_{\xx_T\sim p_T(\xx_T)}\big[ f_\theta^{(T)}(\xx_T) \big] - \mathbb{E}_{\xx_0\sim p_d}\big[\log p_\theta^{(0)}(\xx_0)\big].
\end{align*}
\end{proof}

\subsection{Detailed proof of connections to DRL (Section 3.2)}
The expected recovery likelihood is
\[
\mathbb{E}_{\xx\sim p_d,\Tilde{\xx}\sim p_d^{(\sigma)}(\Tilde{\xx})}\log p_\theta(\xx|\Tilde{\xx}).
\]
Since $p^{(\sigma)}(\Tilde{\xx}|\xx)$ and $p_d(\xx)$ are independent of parameter $\theta$, the objective is equivalent to minimizing
\begin{align*}
    & -\mathbb{E}_{p_d(\xx)p^{(\sigma)}(\Tilde{\xx}|\xx)}\bigg[\log \frac{p_\theta(\xx|\Tilde{\xx})p_\theta^{(\sigma)}(\Tilde{\xx})}{p_d(\xx|\Tilde{\xx})p_d^{(\sigma)}(\Tilde{\xx})} - \log \frac{p_\theta^{(\sigma)}(\Tilde{\xx})}{p_d^{(\sigma)}(\Tilde{\xx})} \bigg]\\
    =&\mathbb{E}_{p_d(\xx)p_\sigma(\Tilde{\xx}|\xx)}\bigg[\log \frac{p_d(\xx,\Tilde{\xx})}{p_\theta(\xx,\Tilde{\xx})}\bigg] - \mathcal{D}_{KL}(p_d^{(\sigma)}(\Tilde{\xx}),p_\theta^{(\sigma)}(\Tilde{\xx}))\\
    =&\mathbb{E}_{p_d(\xx)p_\sigma(\Tilde{\xx}|\xx)}\bigg[\log \frac{p_d(\xx)p(\Tilde{\xx}|\xx)}{p_\theta(\xx)p(\Tilde{\xx}|\xx)}\bigg] - \mathcal{D}_{KL}(p_d^{(\sigma)}(\Tilde{\xx}),p_\theta^{(\sigma)}(\Tilde{\xx}))\\
    =&\mathbb{E}_{p_d(\xx)p_\sigma(\Tilde{\xx}|\xx)}\bigg[\log \frac{p_d(\xx)}{p_\theta(\xx)}\bigg] - \mathcal{D}_{KL}(p_d^{(\sigma)}(\Tilde{\xx}),p_\theta^{(\sigma)}(\Tilde{\xx}))\\
    =&\mathcal{D}_{KL}(p_d(\xx),p_\theta(\xx)) - \mathcal{D}_{KL}(p_d^{(\sigma)}(\Tilde{\xx}),p_\theta^{(\sigma)}(\Tilde{\xx})).
\end{align*}

\subsection{Proof of Proposition 2 (Section 3.3)}
When $t$ is small and by the first-order Taylor approximation, we can write 
\[
f_\theta^{(t)}(\xx) = f_\theta(\xx) + t \big[\frac{d}{\diff t}f_\theta^{(t)}(\xx)\big]|_{t=0} + {o}(t),
\] 
The corresponding DCD objective becomes
\begin{align*}
    &\mathcal{L}_{DCD}^{(VE)}(\theta) =
    \mathbb{E}_{\xx_0\sim p_d, \xx_t \sim p(\xx_t|\xx_0)}[f_\theta^{(t)}(\xx_t)] - \mathbb{E}_{\xx_0\sim p_0}[f_\theta(\xx_0)]\\
    &=\mathbb{E}_{p_t(\xx_t)}[f_\theta^{(t)}(\xx_t) - f_\theta(\xx_t)] + \mathbb{E}_{p_t}[f_\theta(\xx_t)] - \mathbb{E}_{p_d}[f_\theta(\xx_0)]\\
    &=\mathbb{E}_{p_t}t[ \frac{d}{\diff t}f_\theta^{(t)}(x_t)]|_{t=0} + \mathbb{E}_{p_t}[f_\theta(\xx_t)] - \mathbb{E}_{p_d}[f_\theta(\xx_0)]\\
    &=\mathbb{E}_{p_t}\frac{1}{2}\bm{G}^2(0)\big[ \|\nabla_{\xx} f_\theta(\xx)\|^2 + \Delta_{\xx} f_\theta(\xx) \big]+ \mathbb{E}_{p_t}[f_\theta(\xx_t)] - \mathbb{E}_{p_d}[f_\theta(\xx_0)].
\end{align*}

\subsection{Derivation of DCD-VE (Equation 12)}
$d\xx_t = \bm{G}(t)d \bm{w}_t$, the energy evolution is
\begin{equation}
\label{eqn:ve_energy}
    d f_\theta^{(t)}(\xx)/\diff t = \frac{1}{2}\bm{G}^2(t)\big[ \|\nabla_{\xx} f_\theta^{(t)}(\xx)\|^2 + \Delta_{\xx} f_\theta^{(t)}(\xx) \big].
\end{equation}

When $t$ is small and by the first-order Taylor approximation 
\[
f_\theta^{(t)}(x) = f_\theta(x) + t \big[\frac{d}{\diff t}f_\theta^{(t)}(x)\big]|_{t=0} + {o}(t),
\] the corresponding DCD objective becomes
\begin{align*}
    t\mathcal{L}_{DCD}^{(VE)}(\theta) =&
    \mathbb{E}_{\xx_0\sim p_d, \xx_t \sim p(\xx_t|\xx_0)}[f_\theta^{(t)}(\xx_t)] - \mathbb{E}_{\xx_0\sim p_0}[f_\theta(\xx_0)]\\
    &=\mathbb{E}_{p_t(\xx_t)}[f_\theta^{(t)}(\xx_t) - f_\theta(\xx_t)] + \mathbb{E}_{p_t}[f_\theta(\xx_t)] - \mathbb{E}_{p_d}[f_\theta(\xx_0)]\\
    &=\mathbb{E}_{p_t}t[ \frac{d}{\diff t}f_\theta^{(t)}(\xx_t)]|_{t=0} + \mathbb{E}_{p_t}[f_\theta(\xx_t)] - \mathbb{E}_{p_d}[f_\theta(\xx_0)]/
    t \\
    &=\mathbb{E}_{p_t}\frac{1}{2}\bm{G}^2(0)\big[ \|\nabla_x f_\theta(x)\|^2 + \Delta f_\theta(x) \big]\\
    &+ \mathbb{E}_{p_t}[f_\theta(\xx_t)] - \mathbb{E}_{p_d}[f_\theta(\xx_0)].
\end{align*}

\subsection{Backgrounds on Skilling-Hutchison trick}
Skilling-Hutchison's (SH) \citep{Hutchinson1989ASE} stochastic trace estimation trick is a commonly used solution for efficient computation of trace of the Jacobian matrix for high-dimensional problems. In our work, we adapt the SH trick to estimating the trace of Jacobian which appears in \eqref{eqn:dcd_ve}. More precisely, we aim to compute the trace of the Jacobian term
\begin{align}
    \Delta_{\xx} f_\theta(\xx) \coloneqq \nabla_{\xx} \bm{s}_\theta(\xx),
\end{align}
where $\bm{s}_\theta \coloneqq \nabla_{\xx} f_\theta(\xx)$ is the score function of the EBM. The SH estimation uses a stochastic quadratic form to estimate the trace term, i.e.
\begin{align}
    \nabla_{\xx} \bm{s}_\theta(\xx) = \mathbb{E}_{\epsilon\sim p_\epsilon}\epsilon^T \nabla_{\xx} \bm{s}_\theta(\xx) \epsilon = \mathbb{E}_{\epsilon\sim p_\epsilon}(\epsilon^T \nabla_{\xx} \bm{s}_\theta(\xx)) \epsilon.
\end{align}
The distribution $p_\epsilon$ is assumed to be isotropic, i.e. $\mathbb{E}_{\epsilon\sim p_\epsilon}{\epsilon \epsilon^T} = \mathbf{I}$. The multivariate Gaussian distribution is a usual choice for $p_\epsilon$. The vector-Jacobian-product term $\epsilon^T \nabla_{\xx} \bm{s}_\theta(\xx)$ is efficient to implement with deep learning computation framework such as PyTorch with $\mathcal{O}(1)$ memory costs. More precisely, for a data $\xx$, we first compute the score function $\bm{s}_\theta(\xx)$ of the EBM by automatic gradient computation functions of deep learning frameworks such as PyTorch. Then we randomly sample a Gaussian vector and compute the Jacobian-vector product of $\bm{v}^T \bm{s}_\theta(\xx)$. After that, we calculate the final quadratic form $\bm{v}^T \bm{s}_\theta(\xx)\bm{v} = (\bm{v}^T \bm{s}_\theta(\xx))\bm{v}$. However, Though the Skilling-Hutchison trace estimation trick can alleviate the non-linear memory cost problem, frankly speaking, the DCD consumes more GPU memory than MCMC-based methods. From this point of view, the DCD can be understood as a method that trades memory costs for computational efficiency when training EBMs.

\subsection{Algorithm for training time-dependent EBM with DCD-VE}
\begin{algorithm}[]
\SetAlgoLined
\KwIn{dataset $\mathcal{D}=\left\{x_{i}\right\}_{i=1}^{n}$, time-dependent EBM $f_\theta(x,t)$, diffusion process $(F,G)$, perturbation time $\delta$, end timestamp $T$, mini-batch size B.}

\While{not converge}{
Sample time step $t\sim Unif[0,T]$,\\
Sample mini-batch uniformly $\{x_i^{(0)}\}_{i=1}^B \sim \mathcal{D},i=1,..,B$,\\
Diffuse data sample with $x^{(t)}_i \sim p(x^{(t)}_i|x^{(0)}_i)$,\\
Calculate DCD objective $\mathcal{L}_{DCD}(\theta)$(\eqref{eqn:dcd_ve}) with data samples $\{x^{(t)}_i\}_{i=1}^B$,\\
Update $\theta$ with gradient decent according to minimize $\mathcal{L}_{DCD}(\theta)$.
}
\Return{$\theta$.}
\caption{Training time-Dependent EBM with DCD}
\label{alg:dcd}
\end{algorithm}
The available objective $\mathcal{L}_{DCD}(\theta)$ can be $\mathcal{L}_{DCD}^{VE}(\theta)$ or $\mathcal{L}_{DCD}^{VP}(\theta)$ as proposed in previous sections.

\section{More on experiments}
\subsection{Experiment Details on 2D Synthetic Modeling}
\paragraph{Datasets.} We train EBMs on seven 2D datasets: Swissroll, Circles, Rings, Moons, 8Gaussians, 2Spirals and Checkerboard. The code to generate the dataset is adapted from the open source codebase\footnote{https://github.com/wgrathwohl/LSD}. 

\paragraph{Model architecture}
We use the multi-layer perceptron (MLP) with 4 layers and 300 hidden units in each layer as the implementation of the energy-based model. We use the Gaussian Error Linear Units (GELU) \citep{Hendrycks2016GaussianEL} as the activation function. 

\paragraph{Hyper-parameters for DCD-VE.} We use the one-step DCD-VE (equation \eqref{eqn:dcd_ve}) for implementation. We use $t=0.0005$ and $G(0)^2 = 1$. We train all models (with different methods) with the same hyper-parameters: the optimizer is Adam optimizer with $\beta = (0.9, 0.99)$. The batch size is 1000, the learning rate is 0.001 and the number of training iterations is 5000. 
For ablation training methods, i.e. CD and PCD. For CD, we use 0.001 to be the step size of Langevin dynamics. The number of iterations of the Langevin dynamics is set to be 10. For PCD, we use a replay buffer with a size of 10000. The Langevin dynamic step size is set to be 0.001 and the number of MCMC steps is 20. We set the update frequency of the replay buffer to be 5\%, which follows the setting of \citep{du2019implicit}. 

\paragraph{Evaluation metric.} We compute the score-matching loss over the training data as the evaluation metric. The score matching loss is defined with 
\begin{align}
    \operatorname{L}(\theta) \coloneqq \mathbb{E}_{\xx \sim p_d} \bigg[ \frac{1}{2}\|\nabla_{\xx} f_\theta(\xx)\|_2^2 + \Delta_{\xx} f_\theta(\xx) \bigg].
\end{align}
So the smaller the SM loss is, the better the learning performance of the EBM. 

\subsection{Details on image denoising}
In this experiment, we train EBM with CD and DCD-VE on four image datasets for denoising: CIFAR10, SVHN, MNIST, and the FashionMNIST datasets. 

\paragraph{Model architecture.} We use the Wide ResNet \citep{zagoruyko2016wide} with the Sigmoid-weighted Linear Units (SiLU) \citep{Elfwing2017SigmoidWeightedLU} activations and no normalization as the implementation of the energy-based model. For MNIST and the FashionMNIST model, we set the depth to 16 and the widen factor to 8. For the CIFAR10 and SVHN datasets, we set the depth to 28 and the widen factor to 10. 

\paragraph{Training details.} 
We first pre-process the data to scale the range of an image to $[-1,1]$.
In order to let the EBM learn the denoising ability of data samples, we pre-process the training data by adding a Gaussian noise of amount $\sigma = 0.3$. We use the Adam optimizer \citep{Kingma2014AdamAM} with $\beta_0 = 0.9$ and $\beta_1 = 0.99$ and learning rate $0.0002$. For the DCD-VE training algorithm, we set the diffusion strength $t=0.018$ and $G(0)^2 = 1$. To make a fair comparison, we set the step size of the Langevin dynamics also to be $0.018$. For CD, we use one-step of Langevin dynamics for implementing CD. 

\paragraph{Evaluation metric.} To evaluate the denoising performance of trained EBM, we use the trained EBM to denoise noisy images which are added Gaussian noise with three levels: $\sigma=0.3$, $\sigma=0.6$ and $\sigma=0.9$. 

\subsection{Details on image generation}
We train time-dependent EBM with the EDM \citep{karras2022edm} forward diffusion which is a special instance of VE diffusion \eqref{eq:ve_forward}, for which the $g(t)=t$. 

Samples of $\xx_t$ are cheap to obtain by adding Gaussian noise to data samples $\xx_0\sim p_d$. We randomly choose a time $t\sim\operatorname{LogNormal}(t; -1.2, 1.2)$ following the same setting as the EDM model and draw samples with 
$$\xx_t = \xx_0 + \sigma(t) \epsilon, \epsilon\sim \mathcal{N}(\epsilon, \bm{0}, \bm{I}).$$
Here $\xx_0\sim p_d$ denotes a data sample and $\epsilon$ is a standard Gaussian vector of the same size as $\xx_0$. Then we slightly diffuse $\xx_t$ to $\xx_{t+\Delta t}$. This can be done by adding another Gaussian noise of variance $\sqrt{\sigma(t+1)^2-\sigma(t)^2}$ and calculate DCD with $\xx_{t+\Delta t}$ and $\xx_t$. 
\paragraph{Network architecture.} We adopt a UNet encoder from the VP architecture of EDM model \citep{karras2022edm}. We add an additional SiLU non-linearity to the layer before the last pooling layer. 
\paragraph{Sampling Method.} We adapt the Heun sampling algorithm from \citet{karras2022edm} for sampling from time-dependent EDM. We discretize the noise levels from 0.01 to 80.0 to 18 time-stamps with the same strategy of \citet{karras2022edm}.

\end{document}